\documentclass[letterpaper]{article} 
\usepackage{aaai2026}  
\usepackage{times}  
\usepackage{helvet}  
\usepackage{courier}  
\usepackage[hyphens]{url}  
\usepackage{graphicx} 
\usepackage{natbib}  
\usepackage{caption} 
\usepackage{algorithm}
\usepackage{algorithmic}

\usepackage{newfloat}
\usepackage{listings}
\DeclareCaptionStyle{ruled}{labelfont=normalfont,labelsep=colon,strut=off} 
\floatstyle{ruled}
\newfloat{listing}{tb}{lst}{}
\floatname{listing}{Listing}
\usepackage{amsmath}
\usepackage{amsthm}
\usepackage{amssymb}
\usepackage[]{mdframed}
\newcommand\defeq{\mathrel{\overset{\makebox[0pt]{\mbox{\tiny def}}}{=}}}
\newtheorem{definition}{Definition}
\usepackage{csquotes}

\newtheorem{theorem}{Theorem}
\newtheorem{corollary}{Corollary}[theorem]

\nocopyright

\title{Formalizing Kantian Ethics: $\mathcal{F}$ormula of the $\mathcal{U}$niversal $\mathcal{L}$aw $\mathcal{L}$ogic ($\mathcal{FULL}$)}
\author {
    Taylor Olson
}
\affiliations{
    Department of Computer Science, University of Iowa\\
    taylolson@uiowa.edu
}

\usepackage{bibentry}

\begin{document}

\maketitle

\begin{abstract}
  The field of machine ethics aims to build \textit{Artificial Moral Agents (AMAs)} to better understand morality and make AI agents safer. To do so, many approaches encode human moral intuition as a set of axioms on actions e.g., do not harm, you must help others. However, this introduces (at least) two limitations for future AMAs. First, it does not consider the agent's purposes in performing the action. Second, it assumes that we humans can enumerate our moral intuition. This paper explores formalizing a moral procedure that alleviates these two limitations. We specifically consider Kantian ethics and present a multi-sorted quantified modal logic we call the \textit{$\mathcal{F}$ormula of the $\mathcal{U}$niversal $\mathcal{L}$aw $\mathcal{L}$ogic} ($\mathcal{FULL}$). The $\mathcal{FULL}$ formalizes Kant's first formulation of the categorical imperative, the Formula of the Universal Law (FUL), and concepts such as causality and agency. We demonstrate on three cases from Kantian ethics that the $\mathcal{FULL}$ can reason to evaluate agents' actions for certain purposes without built-in moral intuition, given that it has sufficient (non-normative) background knowledge. Therefore, the $\mathcal{FULL}$ is a contribution towards more robust and autonomous AMAs, and a more formal understanding of Kantian ethics.
\end{abstract}


\section{Introduction}

To better understand morality and make AI agents safer, the field of machine ethics aims to build \textit{Artificial Moral Agents (AMAs)}. These are AI agents that can reason about concepts such as moral obligation and prohibition, and use this understanding to evaluate the world and guide their own actions. It has been argued that machine ethics must be done top-down  rather than purely bottom-up to avoid making the fallacious inference from Is to Ought \cite{talat2022machine,OLSON202469}. That is, we must provide AI agents with a set of moral axioms and methods for reasoning from them to yield moral norms for concrete situations. But the critical question is then: \textit{what are the moral axioms?}

Current approaches encode a set of moral axioms pre/proscribing certain actions e.g., do not harm, do not lie, you must help others \cite{anderson2005medethex,bringsjord2014nuclear,malle2017networks,serramia2018exploiting,olson2023mitigating}. However, this provides AMAs with only a weak sense of morality. First, the axioms evaluate only actions (action-centric) and not the agent's purpose in performing the action (agent-centric); thus, for example, they cannot distinguish between reducing pollution out of duty and doing so for public relations. Second, it assumes that we humans can enumerate our moral intuition a priori, when further judgment is often needed to resolve conflicts and exceptions in future situations i.e., the problem of value alignment. For example, without further intuition built in, it would evaluate a surgeon cutting into their willing patient as impermissible given that it causes harm. Third, given that the AMAs cannot critique the moral axioms, they cannot participate in moral discourse about the actions of such axioms. Because of these limitations, current approaches lack the moral robustness and autonomy necessary for Artificial Moral Agents. 

In this paper we explore Kantian ethics \cite{ebels2011kantian} as an alternative foundation for AMAs. We specifically formalize Kant's first formulation of the categorical imperative (CI), the \textit{Formula of the Universal Law (FUL)}, as a moral axiom. The FUL holds that you must \textit{``act only in accordance with that maxim through which you can at the same time will that it should become a universal law''} \cite{kant1998immanuel}. The FUL is appealing as it grounds moral evaluation in equality via universalization. The FUL is also agent-centric and is thus capable of evaluating actions and their purposes. Most important for machine ethics, it is also quite formal as it does not pre/proscribe any specific action. Thus, we do not need to enumerate our moral intuition for it to perform moral evaluation. Therefore, formalizing the FUL will yield more robust and autonomous AMAs.

Researchers have recently made progress on formalizing Kant's second, less formal formulation, the Formula of Humanity \cite{lindner2018formalization}. However, as they state, ``the formalization of the first formulation in effect remains an open problem'' (with the exception of one recent paper \cite{singh2022automated}). In this paper we make progress on this problem by presenting a quantified multi-sorted quantified modal logic we call the \textit{$\mathcal{F}$ormula of the $\mathcal{U}$niversal $\mathcal{L}$aw $\mathcal{L}$ogic}, or $\mathcal{FULL}$. With the $\mathcal{FULL}$ we contribute a formalization of the FUL that: 1) Evaluates actions done for certain purposes (agent-centric), rather than only actions (action-centric); 2) Requires less moral intuition to be encoded by humans; 3) Formally demonstrates core concepts of Kantian ethics.

We start by providing background on Kantian ethics in Section \ref{sec:background}. Next, in Section \ref{sec:formalism} we introduce the $\mathcal{FULL}$. We describe its syntax, proof calculus, and the semantics of its theories of action, causality, agency, and moral evaluation. Next, in Section \ref{sec:demonstration} we demonstrate with formal proofs that the $\mathcal{FULL}$ can reason to produce moral norms. Then, in Section \ref{sec:cc-cw} we illustrate how it formalizes Kant's ontology of contradictions and duties. Lastly, we discuss related and future work.

\section{Background}\label{sec:background}

Immanuel Kant was an 18th century German philosopher whose works spanned many fields such as ethics, epistemology, and metaphysics. His ethical theory falls under \textit{deontology} as it focuses on reasons and principles of action, rather than their external consequences as in \textit{consequentialism}, or virtuous nature as in \textit{virtue ethics}. 

In the \textit{Groundwork of the Metaphysics of Morals} \cite{kant1998immanuel}, Kant aims to discover the laws that govern the will (practical reason) as laws of logic govern belief (theoretical reason). In other words, he aims to discover our moral axioms. He ends up deriving three different formulations of a single axiom called the \textit{categorical imperative}. The first formulation of the categorical imperative is called the \textit{Formula of the Universal Law (FUL)}. It is the most formal of the three and thus best serves our purposes in machine ethics. This law holds that you ought to:
\begin{displayquote}
``Act only in accordance with that maxim through which you can at the same time will that it should become a universal law'' \cite{kant1998immanuel}.
\end{displayquote}

An agent's \textit{will} is the faculty that deliberates and decides how to act. It is the mind in its practical state (reasoning about what to do), rather than its theoretical state (reasoning about what to believe). A \textit{maxim} is then an agent's subjective principle of the will. Maxims take the form, ``I will do $\langle$Action$\rangle$ for $\langle$Purpose$\rangle$.''

To clear up a common misconception, by ``\textit{can} will that your maxim should be become a universal law'', Kant does not mean ``want.'' The Kantian agent does not determine if they would \textit{want} other agents to act as they do, as they would do under a Golden Rule. Instead, the agent determines if they \textit{can}, \textit{without contradiction}, imagine a world in which everyone acts on their maxim.

For example, imagine that Karli intends to travel but needs money. Karli is considering making a promise to repay her friend Jan with no intention of returning the money. The FUL commands Karli to imagine a world in which everyone false promises to get money, and to not perform her action if it would be contradictory in this world. This imagined world is often called the \textit{World of the Universal Law (WUL)}.

Philosophers are divided on what is (or should be) meant by ``contradiction.''  Various interpretations have been offered, but they are usually categorized as either \textit{logical contradiction interpretations (LCI)}, \textit{teleological (TCI)}, or \textit{practical (PCI)}. The LCI holds that contradictions occur when one cannot logically conceive of their action in the WUL. For example, Karli cannot even conceive of promises in the WUL, as the practice of promising would not exist in a world where everyone false promised. However, the LCI on its own does not handle natural actions such as murder, given that their existence does not depend on any conventions, and thus can still be conceived of in the WUL. Therefore, for natural actions, and various maxims of omission, the LCI must introduce a different kind of contradiction.

The teleological contradiction interpretation (TCI) instead posits that rational agents possess innate purposes and ground contradictions in that which does not align these purposes. For example, murder would be wrong because in the WUL it contradicts our natural purpose of happiness.

Lastly, the practical contradiction interpretation (PCI), championed by O'Neill \cite{o2013acting} and Korsgaard \cite{korsgaard1996creating}, grounds contradictions in effectiveness. Under the PCI, the FUL asks the agent to determine if their action would still be an effective means of achieving their purpose in the WUL. If it becomes ineffective, then it is morally impermissible. For example, Karli's false promise to Jan would be an ineffective means of getting money in the WUL, as Jan would not believe it in a world where everyone who needs money false promises. Importantly, the PCI also handles natural actions. For example, murdering for money would be ineffective in the WUL given that everyone else would just murder you for this money as well.

Here we formalize the FUL under a practical contradiction interpretation, as it handles notoriously difficult cases (e.g., violent actions) without introducing ad hoc contradiction types like the LCI, or positing teleological assumptions like the TCI. Restated under a PCI, the FUL holds that: \textit{acting on your maxim is permissible if and only if the action of your maxim is an effective means for achieving the purpose of your maxim in the world of the universal law}. To provide a better moral procedure for Artificial Moral Agents and a formal interpretation of Kantian ethics, we formalize the FUL and necessary concepts of agency in a modal logic. We describe this formalism next.

\section{Formalism}\label{sec:formalism}

We formalize the FUL in a multi-sorted quantified modal logic we call the $\mathcal{F}$ormula of the $\mathcal{U}$niversal $\mathcal{L}$aw $\mathcal{L}$ogic ($\mathcal{FULL}$). In Figure \ref{fig:languageFUL} we define the syntax of $\mathcal{FULL}$ in BNF including a system of sorts $S$, a grammar for terms $t$, and a grammar for sentences $\phi$. We assume a signature for function and predicate symbols of the domain being modeled. The proof calculus for $\mathcal{FULL}$ is based on natural deduction \cite{jaskowski1934rules} and thus includes all elimination and introduction rules for first-order logic and additional rules for modal operators. We list rules for modal operators at the bottom of Figure \ref{fig:languageFUL}. We denote syntactic entailment as $\vdash_{\mathcal{FULL}}$.

\begin{figure*}[t]
\begin{mdframed}
\centering

\textbf{Syntax}

\begin{alignat*}{2} & S ::= \ &&Object \ | \ Agent \ | \  Action \\
& t ::= &&x : S \ | \ c: S \ | \ f(t_1,...,t_n) \\
& \phi ::= && \top \ | \ \bot \ | \ \neg \phi \ | \ \phi_1 \land \phi_2 \ | \ \phi_1 \lor\phi_2 \ | \ \phi_1 \rightarrow \phi_2 \ | \  \phi_1 \leftrightarrow \phi_2 \ | \ \forall x:S.\phi \ | \ \exists x: S. \phi \ | \  t_1 = t_2 \\
& && Does(\langle Agent \rangle,\langle Action \rangle) \ | \ Wills(\langle Agent \rangle,\phi) \ | \ For(\phi_1,\phi_2) \ | \ Causes([\neg]Does(\langle Agent \rangle,\langle Action \rangle),\phi) \\
& &&Perm(For(\phi_1,\phi_2)) \ | \ Imp(For(\phi_1,\phi_2)) \ | \ Obl(For(\phi_1,\phi_2))
\end{alignat*}

\vspace{3mm}
\textbf{Rules of Inference}

\begin{align*}
\frac{Causes(\phi_1,\phi_2)}{\phi_1\rightarrow \phi_2} \quad [R_1]
\quad \quad 
\frac{Wills(a,\phi)}{\neg Wills(a,\neg\phi)} \quad [R_2]
\end{align*}

\begin{align*}
\frac{\phi_1 \rightarrow \phi_2}{Wills(a,\phi_1)\rightarrow Wills(a,\phi_2)} \quad [R_3] 
\quad \quad 
\frac{}{\exists \psi \in \mathcal{FULL}.Wills(a,\phi) \rightarrow \psi}  \quad [R_4]
\end{align*}

\begin{align*}
\frac{For(\phi_1,\phi_2)}{\phi_1 \land Wills(\alpha(\phi_1),\phi_2)} \quad [R_5]
\quad \quad 
\frac{For(\phi_1,\phi_2)}{
Causes(\phi_1,\phi_2)} \quad [R_6]
\end{align*}

\caption{The syntax of the Formula of the Universal Law Logic ($\mathcal{FULL}$), specified in Backus–Naur Form, and its rules of inference. We assume a first-order signature for function and predicate symbols of the domain being modeled. [] indicates optional components.}
\label{fig:languageFUL}
\end{mdframed}
\end{figure*}

We describe the components of $\mathcal{FULL}$ in the next sections. To help illustrate, we continue with our running example of false promising. Consider constants $karli$ and $jan$ of sort $Agent$, the function $falsePromise$ that takes an agent and returns a term representing a $FalsePromise$ to that agent (where $FalsePromise$ is a sub-sort of $Promise$ and further $Action$), and the unary predicate $HasTravelMoney$ on the set of agents. 

\begin{align*}
& karli, jan: Agent \\
& falsePromise: Agent \rightarrow FalsePromise \\
& HasTravelMoney(\langle Agent\rangle) 
\end{align*}

\subsection{Actions and Causal Laws}
We formalize an action theory with the first-order predicate $Does$ over agents and actions. We assume actions are represented as non-atomic terms.

\begin{definition}[Does/Omits]\label{def:does-omits}
Where $a:Agent$ and $b:Action$, we represent the doing of $b$ by $a$ with the first-order predicate $Does(a,b)$. The omission of $b$ by $a$ is then represented as $\neg Does(a,b)$. These are to be read as ``agent $a$ does/omits action $b$,'' respectively.

The accessor functions $\alpha$ and $\beta$ yield the agent and action terms of a $Does$ statement, respectively. 
\end{definition}

For example, $karli$ false promising $jan$ is represented as $Does(karli,falsePromise(jan))$. Where $\phi$ is this $Does$ statement, $\alpha(\phi) = karli$ and $\beta(\phi)=falsePromise(jan)$.

Next, we formalize causal laws with a modal operator $Causes$, as causality is central to moral theory.

\begin{definition}[Causal Law]\label{def:causal-law}
Where $a : Agent$, $b: Action$, and $\phi$ is a formula, causal laws are represented as:
\begin{align*}Causes([\neg]Does(a,b),\phi).\end{align*}
This is to be read as ``agent $a$ doing/omitting action $b$ causes $\phi$ to be true.''
\end{definition}

For example, false promising to repay money being a cause of having money to travel is represented as the causal law below.
\begin{align*}
\forall a_1,a_2:Agent.Causes(Does(a_1,falsePromise(a_2)),\\
HasTravelMoney(a_1)).
\end{align*}

Note that by this definition omissions have causal status, which is still a topic of  philosophical debate \cite{clarke2012omission}. However, it seems clear that omitting to help a drowning child has moral import. Accepting examples such as this requires accepting that omission has causal status.

While providing a full semantic account of causality here is out of scope, we do provide a working proof-theoretic definition of causality via inference rule $R_1$ that serves our purposes of illustrating the FUL.
\begin{align*}R_1: Causes([\neg]Does(a,b),\phi)\vdash_{\mathcal{FULL}}\\
[\neg] Does(a,b)\rightarrow \phi\end{align*}

Intuitively, rule $R_1$ holds that if doing/omitting an action causes some state of affairs to be true, then when it is true that you do/omit that action, those state of affairs are also true. We note that this definition is consistent with (though weaker than) contemporary model-theoretic definitions of causality \cite{halpern2016actual}, and future work will be adapting such accounts for the $\mathcal{FULL}$.

\subsection{Willing}
An agent's will is the faculty that deliberates and decides how to act i.e., the thing that \textit{does} practical reasoning. It has as its object a purpose, or a state of affairs the agent aims to realize. We formalize this with the modal operator $Wills$.

\begin{definition}[Willing]\label{def:willing}
Where $a : Agent$ and $\phi$ is a formula, $a$ willing $\phi$ is represented as:
\begin{align*}Wills(a,\phi).\end{align*}
This is to be read as ``agent $a$ wills the truth of $\phi$.''
\end{definition}

For example, $karli$ willing to have money to travel is represented as: $Wills(karli,HasTravelMoney(karli))$.

The semantics of willing stems from rules $R_2,R_3,$ and $R_4$. We describe each next.

\begin{definition}[$R_2$: Consistency of the Will]\label{def:rule-R_1}
Willing something entails that you do not will the opposite.
\begin{align*} R_2:Wills(a,\phi)\vdash_{\mathcal{FULL}} \neg Wills(a,\neg\phi)\end{align*}
\end{definition}
Rule $R_2$ is a consistency axiom common to modal logic \cite{modallogicSEP}.

\begin{definition}[$R_3$: Monotonicity of the Will] Willing is closed under entailment i.e., willing a fact entails that you will all that fact entails.
\begin{align*}
    R_3: \phi \rightarrow \psi \vdash_{\mathcal{FULL}} Wills(a,\phi)\rightarrow Wills(a,\psi)
\end{align*}
\end{definition}
Rule $R_3$ is a rule of monotonicity and common to modal logic as well. An agent cannot be committed to achieving some state of affairs, know what facts this entails, and not also be committed to those facts.

\begin{definition}[$R_4$: Necessities of the Will] \label{def:R_3-necessities-of-the-will}
The will has necessary requirements.
\begin{align*}R_4: \ \vdash_{\mathcal{FULL}}\exists \psi \in \mathcal{FULL}.Wills(a,\phi) \rightarrow \psi \end{align*}
\end{definition}

Rule $R_4$ holds that willing necessitates certain facts. For example, a rational agent must be free, alive, and so on in order to will anything at all. We do not provide a complete set here, but we do argue that willing necessitates being alive, formalized as $R_4'$ below.
\begin{align*}
    R_4': \quad \vdash_{\mathcal{FULL}} Wills(a,\phi) \rightarrow Alive(a)
\end{align*}

\subsection{Maxims} \label{sec:maxim}

A \textit{maxim} is a principle of practical reasoning. Maxims take the form $\langle \phi_1,\phi_2 \rangle $, where $\phi_1$ is the action that will be done or omitted and $\phi_2$ is the purpose for which $\phi_1$ is done/omitted.\footnote{Philosophers have also often considered a third component $C$ as a condition on which $\phi_1$ will be done/omitted for $\phi_2$ \cite{o2013acting}. This is consistent with contemporary work in conditional deontic logic \cite{gabbay2013handbook}. Though we agree that context plays a role in moral reasoning, including it as a component in maxims introduces issues of tailoring maxims to yield desired permissibility judgments \cite{herman1993practice}. Future work will be exploring how to safely put contextual information into the maxim.} We formalize maxims with the modal operator $For$, as described below.

\begin{definition}[Maxim] \label{def:maxim}
Where $a:Agent$, behavior $\phi_1$ is a (possibly quantified and/or negated) $Does$ statement, purpose $\phi_2$ is a first-order formula, maxims are formalized with the modal operator \textit{For}, written as:
\begin{align*}For(\phi_1,\phi_2).\end{align*}

This is to be read as ``I will $\phi_1$ for (to achieve) $\phi_2$.''
\end{definition}

For example, $karli$ false promising $jan$ in order to have money to travel would be represented as below.
\begin{align*}
For(Does(karli,falsePromise(jan)),\\
HasTravelMoney(karli))
\end{align*}

We capture the meaning of maxims with rule $R_5$ and $R_6$, described below.
\begin{align*}
 R_5: For(\phi_1,\phi_2) \vdash_{\mathcal{FULL}} \phi_1\land Wills(\alpha(\phi_1),\phi_2)
\end{align*}

Rule $R_5$ holds that doing/omitting an action for some purpose entails that you do/omit that action and you will that purpose. Thus, we relate the modal, purposeful concept of doing/omitting with the first-order, non-purposeful concept of doing/omitting and the concept of willing.
\begin{align*}R_6: For(\phi_1,\phi_2) \vdash_{\mathcal{FULL}} Causes(\phi_1,\phi_2)
\end{align*}

Rule $R_6$ holds that willing a maxim entails that the action is a cause of the purpose. A rational agent willing an end conceptually contains their willing of the necessary means, and thus willing the corresponding maxim. Therefore, acting on this maxim entails that their action/omission is a means (cause) for the end (purpose), insofar as they are rational. In other words, if you do/omit some action to attain some purpose, then, given that you are a rational agent, you must take this doing/omitting to be a cause of that purpose.

\subsection{Universal Laws}
The Formula of the Universal Law takes in maxims and yields moral norms via a universalization step. This step takes as input a specific agent's maxim, which is of the form: ``I will do/omit behavior $\phi_1$ for purpose $\phi_2$,'' and outputs a universal law holding that $\phi_1$ is the universal means of achieving $\phi_2$. That is, everyone with a purpose equivalent to $\phi_2$, does/omits an action equivalent to $\phi_1$. Again, this hypothetical world is often called the \textit{World of the Universal Law (WUL)}.

So, how do we \textit{formally} determine when universalized behaviors and purposes are relatively equivalent to their specific counterparts? It can't be strict equivalence, requiring the same objects used, object acted on, time done, location, and so on. For example, the universal law of $karli$'s maxim cannot be that everyone false promises $jan$ specifically, at the same time and location. Instead we must determine some equivalence between the relevant type description of concepts in the maxim. For example, not false promising $jan$ specifically, but false promising some agent that serves as an equivalent substitution for $jan$, relative to the agent performing the action. Others have pointed out the challenges involved in automating this process \cite{o2013acting,herman1993practice}. We provide a working solution as a 
syntactic quantified variable substitution in definition \ref{def:ul}.

\begin{definition}[Universal Law] \label{def:ul}

Let $C(\phi)$ denote the set of constant symbols within formula $\phi$. Given agent $a$'s maxim $M=For(\phi_1,\phi_2)$, the Universal Law of $M$ is created via the transformation $UL(M)$ given the following structures.

\begin{itemize}
    \item An infinite set of variables $V$ partitioned into sets for each respective sort. E.g., $V= V_{Object} \cup V_{Agent} \cup V_{Action}$;
    \item A finite set $T_{\phi_2}=\{a\} \cup C(\phi_2)$. E.g., $\{karli:Agent\}$. I.e., the agent and all constants within their purpose;
    \item A finite set $T_{\phi_1}=C(\phi_1)\setminus T_{\phi_2}$. E.g., $\{jan:Agent\}$. I.e., all constants within the agent's behavior, excluding the agent and those in the their purpose;
    \item A one-to-one mapping $\sigma:T_{\phi_1}\cup T_{\phi_2} \rightarrow V$, respecting sort. E.g., $\sigma  = \{karli \mapsto a_1,jan \mapsto a_2\}$, where $a_1,a_2\in V_{Agent}$;
    \item A finite set $T'_{\phi_2}= \{ c \sigma : c \in T_{\phi_2} \}$. E.g., $\{a_1:Agent\}$. I.e., sorted variables for the agent and each constant within their purpose;
    \item A finite set $T'_{\phi_1}= \{ c \sigma : c \in T_{\phi_1} \}$. E.g., $\{a_2:Agent\}$. I.e., sorted variables for each constant within the agent's behavior, excluding the agent and those in their purpose.
\end{itemize}

The transformation $UL(M)$ is then defined as:
\begin{align*}
    UL(M)=\langle univ \rangle Wills(a\sigma,\phi_2')\rightarrow \langle exist \rangle \phi_1',
\end{align*}
where:
\begin{itemize}
    \item $\langle univ \rangle =$ map $\forall$ onto $T'_{\phi_2}$. E.g., $\forall a_1:Agent$;
    \item $a\sigma$ is the variable that $a$ maps to;
    \item $\phi_2' = \phi_2\sigma$. I.e., $\phi_2$ with all constants substituted for variables. E.g., $\phi_2'=HasTravelMoney(a_1)$;
     \item $\langle exist \rangle =$ map $\exists$ onto $T'_{\phi_1}$. E.g., $\exists a_2:Agent$;
    \item $\phi_1' = \phi_1\sigma$. I.e., $\phi_1$ with all constants substituted for variables. E.g., $\phi_1'=Does(a_1,falsePromise(a_2))$.
\end{itemize}

\end{definition}

Informally, function $UL$ transforms a maxim to a universal law by universally quantifying the agent doing the willing and all other constants within the purpose (set $T_{\phi_2}$), and existentially quantifying all constants within the behavior (set $T_{\phi_1}$). For example, take $M$ to be $karli$'s previously listed maxim of false promising to $jan$. $UL(M)$ is computed via the mapping $\sigma  = \{karli\mapsto a_1,jan\mapsto a_2 \}$:
\begin{align*}
\forall a_1:Agent.Wills(a_1,HasTravelMoney(a_1)) \rightarrow \\
\exists a_2:Agent.Does(a_1,falsePromise(a_2)).
\end{align*}
In English, this universal law states, ``all agents that will to have money to travel, tell false promises to some agent.''

\subsection{Moral Norms}

The FUL outputs moral norms evaluating the input maxims. We formalize moral norms with three modal operators (common in Deontic Logics \cite{deonticlogicSEP}): $Permissible,Impermissible,Obligatory$, abbreviated as $Perm,Imp,Obl$ throughout. We point out that because deontic operators take maxims as objects, they represent the morality of doing actions \textit{for certain reasons}. This agent-centric approach differs from previous deontic logics that evaluate only actions. As we have stated, this is an important step for the field of machine ethics, as AMAs should not only do the right thing, but also for the right reasons. We provide a formal definition of moral norms below.

\begin{definition}[Moral Norm] \label{def:moral-norm}
Where deontic operator $\mathcal{D}\in \{Perm,Imp,Obl\}$, $\phi_1$ is a (possibly quantified and/or negated) $Does$ formula, $\phi_2$ is a first-order formula of $\mathcal{FULL}$, a moral norm is of the following form.
\begin{align*}
& \mathcal{D}(For(\phi_1,\phi_2))
\end{align*}
This can be read as ``doing/omitting $\phi_1$ for purpose $\phi_2$ is permissible/impermissible/obligatory.''
\end{definition}

For example, Karli's maxim being impermissible would be represented as:
\begin{align*}
Imp(For(Does(karli,falsePromise(jan)),\\
HasTravelMoney(karli))).
\end{align*}

We define deontic modalities in relation to each other in standard fashion.
\begin{align*}
&Imp(For(\phi_1,\phi_2)) \defeq \neg Perm(For(\phi_1,\phi_2)) \\ 
&Obl(For(\phi_1,\phi_2)) \defeq Imp(For(\neg \phi_1,\phi_2))
\end{align*}

\subsection{Formula of the Universal Law}

The meaning of deontic operators stems from the Formula of the Universal Law: \textit{it is permissible to act on your maxim if and only if its action is an effective means for its purpose in the World of the Universal Law}.  We formalize the FUL with rule $FUL$ in definition \ref{def:ci}.

\begin{definition}[Formula of the Universal Law] \label{def:ci}
Where $\Gamma$ is a consistent set of formulae of $\mathcal{FULL}$ describing a world, $a:Agent$ with maxim $M$, rule $FUL$ is formalized as below. 
\begin{align*}
    FUL: \quad &\Gamma \vdash_{\mathcal{FULL}} Perm(M) \Leftrightarrow \\
    &\Gamma, Wills(a, UL(M)) \land  Wills(a, M) \nvdash_{FUL} \bot
\end{align*}

Default assumption (negation as failure): if a maxim cannot be proven as permissible, then it is not permissible.
\begin{align*}
\Gamma \nvdash_{\mathcal{FULL}} Perm(M) \Rightarrow \Gamma \vdash_{\mathcal{FULL}} \neg Perm(M)\end{align*}
\end{definition} 

Note that we do not explicitly define the notion of effectiveness from the PCI. This is because effectiveness is actually a secondary concept derived from principles of rational agency \cite{korsgaard1996creating}. A maxim becomes ineffective when the causal law from which it is derived becomes false. This is contradictory for the agent because they are acting on their maxim as well, and thus take their action to be a cause of their purpose. In other words, the agent wills that their maxim is both ineffective and effective. This is the underlying derivation of the PCI.

Our interpretation and formalization of the FUL can be summarized as follows. We assume all agents are rational (formalized with rules for willing). We then formalize the FUL as a negative test consisting of three reasoning steps: 1) The agent wills a world in which the universal law is true: $Wills(a,UL(M))$; 2) The agent wills to act on their maxim in that same world: $Wills(a,M)$; 3) The agent simulates to determine if the previous two steps would result in a contradiction: $\vdash_{FUL} \bot$. If it does result in a contradiction, then by Modus Tollens (MT) from rule $FUL$ and given our default assumption, acting on maxim $M$ is impermissible. \textit{To say otherwise would be to take oneself as an exception, as the only difference between the WUL and one's own world is the universal law}. This is Kant's central moral axiom, and it is quite powerful. From purely formal principles of rational agency, causality, and background knowledge it can produce moral norms. We demonstrate this in the next section.

\section{Demonstration}\label{sec:demonstration}

In this section we use canonical examples from Kantian ethics to demonstrate how the $\mathcal{FULL}$ reasons to moral norms. We focus on critiquing input maxims and thus assume the cognitive process that formulates them. We also introduce function and predicate symbols on the fly, but their signature is made clear by example. We start by demonstrating on the case of false promising.

\subsection{Example 1: False Promising}\label{sec:false-promise}

Let $\Gamma$ be a consistent set of formulae of $\mathcal{FULL}$ describing a world. Let $karli$ be our agent of interest and $jan$ be another agent from this world. Consider $karli$'s maxim $M:$ \textit{``I will make a false promise to repay money to $jan$ in order to get money to travel,''} represented below.
\begin{align*}
    M=For(Does(karli,falsePromise(jan)), \\
    HasTravelMoney(karli))
\end{align*}

Acting on maxim $M$ is intuitively impermissible. Computing this moral norm in $\mathcal{FULL}$ means computing $\Gamma \vdash_{\mathcal{FULL}} Imp(M)$ from rule $FUL$. To do so, we must assume $karli$ possesses common sense and causal knowledge of her world. We formalize three facts necessary for this case below. Assume $B1,B2,B3 \in \Gamma$.

\textbf{B1:} There is another agent that wills travel money.\footnote{Though this is a strong assumption, it is justified by an epistemic lack of $karli$'s. $karli$ can never be completely certain that there are not other agents that share her same pursuits.}
\begin{align*}\exists & a_1: Agent. \\
&a_1 \neq karli \land Wills(a_1,HasTravelMoney(a_1))\end{align*}

\textbf{B2: }Promising to repay money causes getting money to travel only if the promisee believes said promise.
\begin{align*}
\forall & a_1,a_2:Agent.\\
&Causes(Does(a_1,promise(a_2)),  \\
 & \quad \quad \quad \quad HasTravelMoney(a_1)) \rightarrow\\
&BelievesPromise(a_2, promise(a_2))
\end{align*}

\textbf{B3: }Agents believe false promises to repay only if agents in their world do not break their promises to repay.
\begin{align*}
    \forall & a_2 : Agent.\\
    & BelievesPromise(a_2, promise(a_2))\rightarrow \\
    &\forall  a_3,a_4 : Agent.\neg Does(a_3,falsePromise(a_4))
\end{align*}

The universal law of $karli$'s maxim $UL(M)$ is \textit{``everyone who wills to have travel money, makes false promises to repay''} and is formally represented below.
\begin{align*}
    &UL(M)= \\
    & \forall a_1: Agent.Wills(a_1,HasTravelMoney(a_1))\rightarrow \\
    & \exists a_2:Agent.Does(a_1,falsePromise(a_2))
\end{align*}

Next, we formally prove in $\mathcal{FULL}$ that maxim $M$ is impermissible in the world described by $\Gamma$. Again, this works in three steps: 1) will the universal law $UL(M)$, 2) will your maxim $M$, and 3) check for contradiction.

\begin{theorem}$\Gamma \vdash_{\mathcal{FULL}} Imp(M)$\label{thm:false-promising}
\end{theorem}

\begin{proof}
As the FUL is a negative test, we proceed with a proof by contradiction and show that $\Gamma ,Wills(karli,UL(M)) \land Wills(karli,M) \vdash_{\mathcal{FULL}} \bot$. We leave notation $\Gamma \vdash_{\mathcal{FULL}}$ implicit in derivations.

First, consider $karli$ willing $UL(M)$ (step 1). By Modus Ponens (MP) from $B1$ and the definition of $UL(M)$, we can infer that there exists some agent(s) that false promise to repay another agent:
\begin{align*}
 \exists a_1,a_2:Agent.Does(a_1,falsePromise(a_2)).
\end{align*}

Thus, by MP from $R_3$, $karli$ wills a world in which an agent false promises to repay:
\begin{align*}Wills(karli, & \ \exists a_1,a_2:Agent. \\
&Does(a_1,falsePromise(a_2))).
\end{align*}

Thus, by Modus Tollens (MT)  from $B3$ and MP from from $R_3$ we can infer that $karli$ also wills a world where nobody believes promises:
\begin{align*}
Wills(karli,\forall  a_2 &: Agent.
\\ &\neg BelievesPromise(a_2,
promise(a_2))).
\end{align*}

Therefore, by MT from $B2$ and MP from $R_3$, we get that $karli$ also wills a world in which promising to repay is not a means for having money to travel:
\begin{align*}Wills(karli,\forall &a_1,a_2 : Agent.\\
&\neg Causes(Does(a_1,promise(a_2)),\\
& \quad \quad \quad \quad HasTravelMoney(a_1))).\end{align*}

Substituting $karli$ for $a_1$, $jan$ for $a_2$, and the sub-sort $FalsePromise$ for $Promise$, by Universal Elimination we know that $karli$ wills that her own false promise is not a means for her end of having money to travel:
\begin{align*}Wills(karli,\\ 
\neg Causes(&Does(karli,falsePromise(jan)),\\
&HasTravelMoney(karli))).\end{align*}

By $R_2$ and Negation Elimination we can then get that $karli$ does not will this causal law:
\begin{align*}\neg Wills(karli,\\Causes(&Does(karli,falsePromise(jan)),\\
&HasTravelMoney(karli))).\end{align*}

Now consider $karli$ willing to act on her maxim in the WUL (step 2): $Wills(karli,M)$. By MP from $R_6$ we know that insofar as $karli$ is rational, her action must be a cause of her purpose. Thus, by MP from $R_3$ we know that in willing to act on her maxim, $karli$ also wills the causal law:
\begin{align*}Wills(karli,Causes(Does(karli,falsePromise(jan)),\\
HasTravelMoney(karli)))).\end{align*}

Thus, we have produced our contradiction.
\begin{align*}\Gamma, Wills(Karli, UL(M)) \land  Wills(Karli, M) \vdash_{FUL} \bot
\end{align*}

Therefore, by MT from rule $FUL$ and given our default assumption we get:
\begin{align*}\Gamma \vdash_{\mathcal{FULL}} \neg Perm(M).
\end{align*}

Finally, by definition we get:
\begin{align*}
\Gamma \vdash_{\mathcal{FULL}} Imp(M).
\end{align*}

$Karli$, in willing a world in which everyone false promises to repay in order to get travel money, has willed a world where such false promises are an ineffective means of getting said money. In also willing to act on her maxim in this world, she has willed that such false promises \textit{are} an effective means of getting said money. Thus, this world is contradictory, and therefore acting on her maxim in her own world is morally impermissible.
\end{proof}

Next, we consider a maxim of murder (inspired by \cite{korsgaard1996creating}). As we described in Section \ref{sec:background}, under a logical contradiction interpretation the FUL struggles with such natural actions as they do not become inconceivable in the WUL. However, we illustrate via the $\mathcal{FULL}$ how a practical contradiction interpretation elegantly handles natural actions.

\subsection{Example 2: Murder}\label{sec:murder}

Let $\Gamma$ be a consistent set of formulae of $\mathcal{FULL}$ describing a world. Let $karli$ be our agent of interest from this world. Consider $karli$'s maxim $M$: ``I will murder the other job candidate in order to be hired over them and securely possess this job.'' Where $j:Job$ is the job and $jan$ is the other candidate, maxim $M$ is formalized below.
\begin{align*}
    M = For(Does(karli,murder(jan)),\\
    HiredOver(karli,j,jan) \land \\ SecurelyPoss(karli,j))
\end{align*}

Acting on maxim $M$ is clearly impermissible. Computing this moral norm in $\mathcal{FULL}$ again requires knowledge of common sense and causal facts of the world. We formalize two facts necessary for this case below. Assume $B1$ and $B2 \in \Gamma$.

\textbf{B1:} Murdering other job candidates is a cause of being hired over them and securely possessing the job only if there is not another agent murdering you.
\begin{align*}\forall a_1,a_2:Agent.j_1:Job.
Causes(Does(a_1,murder(a_2)), \\HiredOver(a_1,j_1,a_2) \land \\  SecurelyPoss(a_1,j_1)) \\
\rightarrow \nexists a_3 Does(a_3,murder(a_1))\end{align*}

\textbf{B2:} There is another agent that wills to be hired over $karli$ to securely possess job $j$.\footnote{This is another strong assumption that is justified by the fact that $karli$ can never be completely certain that there are no other agents pursuing her same jobs. In fact, $karli$ can be quite certain that $jan$ wills to be hired over her.}
\begin{align*}\exists a_x:Agent.
&a_x\neq karli \land \\&Wills(a_x,HiredOver(a_x,j,karli)\land \\ & \quad \quad \quad \quad \quad SecurelyPoss(a_x,j))\end{align*}

The corresponding universal law of $karli$'s maxim $M$ is ``for all agents, if there is another candidate that they will to have a secure job over, then they murder them.''
\begin{align*}
     UL(M)= & \forall a_1: Agent. a_2.Agent. j_1:Job. \\
     & Wills(a_1,HiredOver(a_1,j_1,a_2) \land \\ & \quad \quad \quad \quad   SecurelyPoss(a_1,j_1))\\ 
     & \rightarrow 
     Does(a_1,murder(a_2))
\end{align*}

Next we formally prove in $\mathcal{FULL}$ that $karli$'s maxim $M$ is impermissible.

\begin{theorem}$\Gamma \vdash_{\mathcal{FULL}} Imp(M)$
\end{theorem}

\begin{proof}
We again proceed with a proof by contradiction and show that $\Gamma,Wills(karli,UL(M)) \land Wills(karli,M) \vdash_{\mathcal{FULL}} \bot$.

By MP from $B2$ and $UL(M)$ we know that $a_x$ murders $karli$ (or $karli$ cannot be certain that this will not happen).
\begin{align*}
     Does(a_x,murder(karli))
\end{align*}

Given that $karli$ will be murdered, by MT from $B1$ we know that $karli$'s murdering of other job candidates is not a cause of her being hired over these candidates to securely possess the job:
\begin{align*}\forall a_2:Agent.
\neg Causes(Does(karli,murder(a_2)),\\ HiredOver(karli,j,a_2) \land \\ 
SecurelyPoss(karli,j)). \end{align*}





Thus, by MP from $R_3$ and this fact and substituting $jan$ for $a_2$, by Universal Elimination we know that $karli$ has willed a world where her murdering of $jan$ does not cause her to be hired for and secure in job $j$:
\begin{align*}
    Wills(karli,\neg Causes(Does(karli,murder(jan)),\\
    HiredOver(karli,j,jan) \land \\SecurelyPoss(karli,j))).
\end{align*}

Therefore, by MP from rule $R_3$ and $R_2$ and Negation Elimination we know:
\begin{align*}
    \neg Wills(karli, Causes(Does(karli,murder(jan)),\\
    HiredOver(karli,j,jan) \land \\SecurelyPoss(karli,j))).
\end{align*}

Now consider $karli$'s willing of maxim $M$. By MP from rules $R_3$ and $R_5$ we get that $karli$ wills that her action \textit{is} a cause of her purpose:
\begin{align*}
Wills(karli, Causes(&Does(karli,murder(jan)),\\
&HiredOver(karli,j,jan) \land \\ &SecurelyPoss(karli,j))).
\end{align*}



Therefore, we have reached a contradiction and given our two assumptions of $karli$ willing her own maxim and its universalization we know the following.
\begin{align*}\Gamma, Wills(Karli, UL(M)) \land  Wills(Karli, M) \vdash_{FUL} \bot
\end{align*}

Finally, by MT from rule $FUL$ and given our default assumption, by definition we get:
\begin{align*}
    \Gamma \vdash_{\mathcal{FULL}} \neg Perm(M)\defeq Imp(M).
\end{align*}

$karli$, in willing a world in which everyone murders other job candidates in order to secure a job, and willing to act on such a maxim herself in this world, wills a world where such murdering is and is not an effective means of securing the job. Thus, this world is contradictory, and therefore acting on her maxim is morally impermissible.
\end{proof}

To illustrate the importance of the $\mathcal{FULL}$ being an agent-centric theory, let us say $karli$ murdered $jan$ by cutting her with a knife and contrast this with a case in which $karli$ is a surgeon cutting into her wiling patient $jan$ for a life-saving surgery. Even given the same harmful action, the $\mathcal{FULL}$ would evaluate it differently given this change of purpose. In this case, cutting a willing patient is still a successful means of saving their life in the World of the Universal Law. That is, this maxim would not produce a contradiction in the WUL. Therefore, by rule $FUL$, surgeon $karli$'s maxim would now be evaluated as permissible. By contrast, current approaches would would either evaluate surgery as impermissible, or require humans to encode further moral intuition in the axioms noting surgery as an exception to harm.

Next, we consider the maxim of omitting aid to others. As we described in Section \ref{sec:background}, the FUL struggles in such cases of omission without introducing different types of contradictions or making teleological assumptions. However, we again illustrate with the $\mathcal{FULL}$ how the FUL handles these cases under a practical contradiction interpretation.


\subsection{Example 3: Never helping others}\label{sec:never-help}

Let $\Gamma$ be a consistent set of formulae of $\mathcal{FULL}$ describing a world. Let $karli$ be our agent of interest. Consider $karli$'s maxim $M:$ \textit{``I will not help others in order to have time for leisure,''} represented below.
\begin{align*}
    M = For(\forall a_x:Agent.\neg Does(karli,help(a_x)),\\
    LeisureTime(karli))
\end{align*}

Acting on maxim $M$ is intuitively impermissible. We formalize two background facts necessary for this case below. Assume $B1,B2 \in \Gamma$.

\textbf{B1:} All agents will time for leisure.
\begin{align*}
    \forall a_1 : Agent.Wills(a_1,LeisureTime(a_1))
\end{align*}

\textbf{B2:} The life of all rational agents requires that some other agent helps them.\footnote{We are finite creatures, and thus cannot do everything by ourselves. We need someone to feed and care for us as infants. We need doctors when we are sick. We require such aid at some times and places, but not universally.}
\begin{align*}
    \forall a_1 :Agent.&Alive(a_1)\rightarrow \\
    &\exists a_2:Agent.a_2\neq a_1 \land Does(a_2,help(a_1))
\end{align*}

The corresponding universal law of $karli$'s maxim $M$ is, ``everyone that wills leisure time does not help others,'' formally represented below.
\begin{align*}
    \forall a_1,a_3:Agent. Wills(a_1,LeisureTime(a_1))\rightarrow \\
    \neg Does(a_1,help(a_3))
    \end{align*}

Next we formally prove in $\mathcal{FULL}$ that maxim $M$ is impermissible in the world described by $\Gamma$.

\begin{theorem}\label{thm:never-helping}
$\Gamma \vdash_{\mathcal{FULL}} Imp(M)$\end{theorem}

\begin{proof}
We again proceed with a proof by contradiction and show: $\Gamma,Wills(karli,UL(M)) \land Wills(karli,M) \vdash_{\mathcal{FULL}} \bot$.

By MP from $R_3$ given $UL(M)$ and $B1$, we get:
\begin{align*}
    Wills(karli,\forall &a_1, a_3 : Agent.\\
    &\neg Does(a_1,action(a_1,help(a_3)))).
\end{align*}

Therefore, by Universal Elimination substituting $karli$ for $a_3$ and re-representing as a negated existential, we can infer that $karli$ wills a world where no agent helps her.  
\begin{align*}
    Wills(&karli, \\ 
    &\nexists a_1 : Agent.Does(a_1,action(a_1,Help(karli))))
\end{align*}

By MT from this fact and $B2$ and MP from $R3$ we know that Karli has willed a world in which she is not alive:
\begin{align*}
Wills(karli,\neg Alive(karli)).
\end{align*}

By MP from $R_3$ and $R_2$ and Negation Elimination we get:
\begin{align*}
 \neg Wills(karli,Alive(karli)).
\end{align*}

However, by $R_4'$, $karli$'s willing of her maxim (really, the willing of anything) necessitates certain facts, specifically that $karli$ is alive. Therefore, by MP from $R_3$, $R'_4$, and $Wills(karli,M)$ we get:
\begin{align*}
    Wills(karli,Alive(karli)).
\end{align*}

We now have our contradiction.
\begin{align*}\Gamma, Wills(Karli, UL(M)) \land  Wills(Karli, M) \vdash_{FUL} \bot
\end{align*}

Therefore, by MT from rule $FUL$ and given our default assumption, willing such a maxim in $karli$'s own world $\Gamma$ is impermissible:
\begin{align*}
    \Gamma \vdash_{\mathcal{FULL}} \neg Perm(M)\defeq Imp(M).
\end{align*}

$karli$, in willing her maxim of not helping others for leisure time and willing this as a universal law, wills that she is alive and that she is not alive. Thus, never helping others is no longer an effective means of having time for leisure in this world, as she is not alive to will time for leisure. Therefore, acting on her maxim is morally impermissible.
\end{proof}

A corollary of obligation follows.

\begin{corollary} $karli$ must help some agent(s).
\begin{align*}
\Gamma \vdash_{\mathcal{FULL}} Obl(For(\exists a_x:Agent. Does(karli,help(a_x)), \\
LeisureTime(karli)))
\end{align*}
\end{corollary}

\begin{proof}
By Theorem \ref{thm:never-helping} we know:
\begin{align*}
    Imp(For(\forall a_x:Agent.\neg Does(karli,help(a_x)),\\
    LeisureTime(karli))) .
\end{align*}

We know: $Obl(For(\phi_1,\phi_2)) \defeq Imp(For(\neg \phi_1,\phi_2))$. Thus, we know:
\begin{align*}
&Obl(For(\neg(\forall a_x:Agent.\neg Does(karli,help(a_x))),\\
    &\quad \quad \quad \quad LeisureTime(karli)))  \\
\equiv \ &Obl(For(\exists a_x:Agent. Does(karli,help(a_x)),\\
    &\quad \quad \quad \quad LeisureTime(karli))).
\end{align*}

Therefore, $karli$ must help some agent(s).
\end{proof}

\section{Types of Contradictions and Duties}\label{sec:cc-cw}

We have demonstrated how the $\mathcal{FULL}$ derives moral norms by deriving contradictions given a theory of causality and rational agency. Note that the contradictions produced in the cases of false promising (example \ref{sec:false-promise}) and murder (example \ref{sec:murder}) differ in form from that of omitting help (example \ref{sec:never-help}). In the former cases, the contradictions involved the purpose \textit{within the maxim} (has travel money and being hired over). In Kantian ethics, such contradictions are termed \textit{contradictions in conception (CC)}\footnote{We note that philosophers have disagreed on this for violent actions such as murder. See \cite{herman1993practice} for an argument that this is instead a contradiction in the will.}, which yield \textit{perfect (or narrow) duties}, as they command exactly what action to do/omit ($karli$ must not false promise nor murder). 

However, in the case of omitting help we derived a contradiction involving a necessary purpose \textit{of the will} (being alive). In Kantian ethics, these are termed \textit{contradictions in the will (CW)}, which yield \textit{imperfect (or wide) duties}, as they command only to sometimes, and to some extent, do/omit some action ($karli$ must help someone at some point). Thus, the $\mathcal{FULL}$ can sort types of contradictions and duties as well. We formalize this below.

\begin{definition}[Contradiction in Conception]\label{def:contradiction-conception}
Given agent $a$ in a world described by $\Gamma$ with maxim $M=For(\phi_1,\phi_2)$, a contradiction in conception occurs when:
\begin{align*}
     \Gamma, Wills(a, M) \land Wills(a, UL(M)) \vdash_{\mathcal{FULL}} \\
     Wills(a,Causes(\phi_1,\phi_2)) \ \land \\
    Wills(a,\neg Causes(\phi_1,\phi_2)).
\end{align*}
\end{definition}

Contradictions in conception yield perfect duties to not $\phi_1$ to the fullest extent possible, given purpose $\phi_2$.

\begin{definition}[Contradiction in the Will]\label{def:contradiction-will}
Given agent $a$ in a world described by $\Gamma$ with maxim $M=For(\phi_1,\phi_2)$, and $\psi$ is a fact necessary to willing: $Wills(a,\phi) \vdash_{\mathcal{FULL}} \psi$, a contradiction in the will occurs when:
\begin{align*}
    \Gamma,Wills(a, UL(M)) \land Wills(a,M) \vdash_{\mathcal{FULL}} \\
    Wills(a,\psi) \land Wills(a,\neg \psi).
\end{align*}

\end{definition}

Contradictions in the will yield imperfect duties that command the agent to adopt the necessary end $\psi$, and therefore to sometimes, and to some extent, not $\phi_1$, given purpose $\phi_2$.

\section{Related Work}

Early work by \cite{kroy1976partial} provides partial formalizations of Kant's first and second formulations utilizing Hintikka's possible-world semantics \cite{hintikka1969deontic}. They present a universalization technique similar to the one we present here. However, they formalize a more trivial Golden Rule: ``Don't do to others what you don't want done to yourself.''

In more recent work, \cite{lindner2018formalization} formalize Kant's second formulation. They utilize causal-agency models \cite{halpern2016actual}, which we plan to explore in future work. However, because Kant's second formulation is less formal than his first, they end up grounding moral evaluation in an informal and subjective notion of affect.


Most related to our work here is  \cite{singh2022automated}, which presents a formalization of Kant's first formulation under a PCI. They utilize dyadic deontic logic \cite{carmo2013completeness} and test their theory in the interactive theorem prover Isabelle \cite{wenzel2008isabelle}. Our work expands on theirs by formalizing the notion of effectiveness from the PCI in a way that better captures its derivation. We fully express the causal nature of effectiveness with defeasible causal laws and reason to contradictions based on formal principles of the will. Our formalism also distinguishes between contradictions in conception and contradictions in the will and can thus sort perfect and imperfect duties.

\section{Discussion and Future Work}

Here we focus on the form of the arguments made from the FUL and have thus presented the $\mathcal{FULL}$ in a proof-theoretic fashion. However, we plan to explore a model-theoretic account of intensional concepts (e.g., intentional actions \cite{Friedenberg2025}) in future work. We also plan to expand the expressiveness of the $\mathcal{FULL}$. As has been noted in the literature \cite{powers2006prospects}, concepts such as promising survive occasional defeats that cannot be formalized with standard quantifiers $\forall$ and $\exists$. Thus, future work will require building in generalized quantifiers that model concepts like few and most \cite{generalizedquantSEP}.

Lastly, we have assumed the process of formulating maxims. This is a sincere limitation because it further assumes that input maxims are properly specified. The tragedy of Oedipus has instead shown us that action descriptions matter for evaluation. For example, if we instead described $karli$'s murdering of $jan$ as ``$karli$ destroys a collection of cells'', which is a true but not properly specific description, then her maxim may be deemed permissible. Philosophers have argued that pinning down these descriptions likely still requires moral intuition and judgment \cite{herman1993practice,o2013acting}. Thus, though the $\mathcal{FULL}$ reduces the need for human moral intuition to be encoded in AMAs, it may not fully remove it. Both Herman and O'Neill presented possible solutions that we plan to explore.

\section{Conclusion}

This paper presents the $\mathcal{F}$ormula of the $\mathcal{U}$niversal $\mathcal{L}$aw $\mathcal{L}$ogic ($\mathcal{FULL}$). The $\mathcal{FULL}$ is a multi-sorted quantified modal logic that formalizes Kant's first formulation of the categorical imperative, the Formula of the Universal Law (FUL), and concepts related to rational agency and causality. Through formal proofs on canonical examples in Kantian ethics we have demonstrated that the FUL serves as a better foundation for automated moral reasoning. First, it does not simply evaluate actions/omissions, but also agents' purposes. Second, it does not require humans to enumerate, a priori, our moral intuition as a set of moral axioms. Nor does it require agents to make the fallacious inference from the judgments of a population to moral norms. Instead, the $\mathcal{FULL}$ reasons top-down from formal, non-normative principles to moral norms. Therefore, this paper contributes towards a more formal understanding of Kantian ethics and improving the robustness of future Artificial Moral Agents.

\section{Acknowledgments}
We thank Kyla Ebels-Duggan for thoughtful discussions on Kantian ethics over the years that helped shape ideas in this work (though these views are not necessarily shared by her).

\bibliography{references}

\end{document}